\documentclass[sigconf]{acmart}
\usepackage{float}
\usepackage{stfloats}
\usepackage{amsthm}
\usepackage{hyperref}
\usepackage{url}
\usepackage{graphicx}
\usepackage{wrapfig}
\usepackage{bbm}
\usepackage{subfigure}
\usepackage{caption}
\usepackage{subcaption}
\usepackage{ragged2e} 
\usepackage{enumitem}
\usepackage{booktabs,makecell, multirow, tabularx}
\AtBeginDocument{%
  }

\newtheorem{lemma}{Lemma}
\setcopyright{acmlicensed}
\copyrightyear{2018}
\acmYear{2018}
\acmDOI{XXXXXXX.XXXXXXX}

\newcommand{\ie}{\emph{i.e., }}
\newcommand{\eg}{\emph{e.g., }}

\newcommand{\wrt}{\emph{w.r.t. }}

\acmConference[Conference acronym 'XX]{Make sure to enter the correct
  conference title from your rights confirmation emai}{June 03--05,
  2018}{Woodstock, NY}
\acmISBN{978-1-4503-XXXX-X/18/06}

\begin{document}

\title{Informative Graph Structure Learning}

\author{Shen Han}
\authornotemark[2]
\orcid{0000-0001-6714-5237}
\affiliation{
\institution{Zhejiang University}
\city{Hangzhou}
\country{China}}
\email{drhanshen@zju.edu.cn}

\author{Zhiyao Zhou}
\authornotemark[2]
\orcid{0009-0005-9291-169X}
\affiliation{
\institution{Zhejiang University}
\city{Hangzhou}
\country{China}}
\email{zjucszzy@zju.edu.cn}

\author{Jiawei Chen}
\authornote{Corresponding author.}
\authornote{College of Computer Science and Technology, Zhejiang University.}
\orcid{0000-0002-4752-2629}
\affiliation{
\institution{Zhejiang University}
\city{Hangzhou}
\country{China}}
\email{sleepyhunt@zju.edu.cn}

\author{Sheng Zhou}
\orcid{0000-0003-3645-1041}
\affiliation{
\institution{Zhejiang University}
\city{Hangzhou}
\country{China}}
\email{zhousheng_zju@zju.edu.cn}

\author{Canghong Jin}
\affiliation{
\institution{Hangzhou City University}
\city{Hangzhou}
\country{China}}
\email{jinch@hzcu.edu.cn}

\author{Hai Lin}
\affiliation{
\institution{China Mobile Communications Group Co.,Ltd}
\city{Hangzhou}
\country{China}}
\email{linhai34@chinaunicom.cn}

\author{Da Zhong Li}
\affiliation{
\institution{China Mobile Communications Group Co.,Ltd}
\city{Hangzhou}
\country{China}}
\email{lidazhong@chinaunicom.cn}

\author{Bingde Hu}
\authornotemark[2]
\affiliation{
\institution{Zhejiang University}
\city{Hangzhou}
\country{China}}
\email{tonyhu@zju.edu.cn}

\author{Can Wang}
\orcid{0000-0002-5890-4307}
\authornotemark[2]
\affiliation{
\institution{Zhejiang University}
\city{Hangzhou}
\country{China}}
\email{wcan@zju.edu.cn}

\renewcommand{\shortauthors}{Shen Han, et al.}
\begin{abstract}
The quality of graph-structured data is fundamental to the success of modern graph analysis techniques such as Graph Neural Networks (GNNs). 
However, real-world graph data is often suboptimal, suffering from issues such as noise and incomplete connections. 
Graph Structure Learning (GSL) has emerged as a promising technique that  adaptively optimizes node connections. 
However, we observe that the effectiveness of GSL often comes at the cost of a dramatic expansion in  edge count,  resulting in significant storage and computational overhead. 

In this work, we reveal that this limitation stems from the prevalent use of similarity-based edge construction, which predominantly connects highly similar neighbors based on their embeddings, introducing substantial structure redundancy. To address this, we propose a novel Informative Graph Structure Learning method (InGSL), which jointly considers both similarity and diversity in edge construction by incorporating a mutual-information-guided learning strategy. 
Notably, InGSL serves as a plug-in module that can be seamlessly integrated into existing GSL frameworks. Through extensive experiments on six representative GSL methods, we demonstrate that InGSL achieves significant performance improvements at a reduced number of edges.
\end{abstract}



\maketitle

\section{INTRODUCTION}
Graph Neural Networks (GNNs)~\cite{kipf2017semisupervised,gilmer2017neural,defferrard2016convolutional,cuiyu} have achieved remarkable success in modeling and analyzing graph-structured data. To further improve their capacity, researchers have developed a wide range of advanced GNN architectures \cite{wu2019simplifying,gprgnn,deng2024polynormer,liu2024scalable}. However, the performance of these model-centric innovations is often limited by potential deficiencies in the underlying graph structure itself. In real-world scenarios,  graph data often exhibits flaws, \eg suffering from noise and missing connections, due to complex and inconsistent data collection processes~\cite{zhou2023opengsl,li2023gslb}.

To address these challenges, Graph Structure Learning (GSL) has emerged as a promising data-centric approach that aims to refine graph structures prior to downstream GNN training. GSL learns to refine node connections and edge weights, yielding improved graph structures for various downstream applications. Early GSL methods directly considered the target graph structure (\ie the adjacency matrix) as learnable parameters~\cite{franceschi2019learning,jin2020graph}, which often incurred optimization challenges due to the large parameter space.  To tackle this, recent research has shifted towards the embedding-based GSL paradigm~\cite{wang2023prose,in2024self,shen2025towards,Zeng0SGD025}, where edges are constructed based on node embedding similarity. Typically, nodes are connected to their most similar peers, thereby increasing graph homophily and improving performance on various downstream tasks.
\begin{figure}
    \centering    \includegraphics[width=0.91\linewidth, trim = 40 15 15 5]{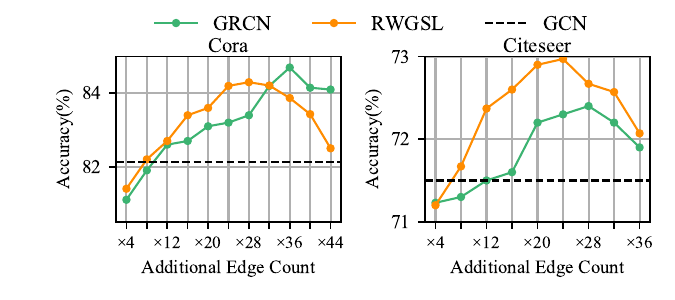}
    \setlength{\belowcaptionskip}{-15pt}
    \caption{Node classification accuracy versus additional edge count of representative GSL methods (GRCN~\cite{yu2021graph}, RWGSL~\cite{shen2025towards}) on the Cora and Citeseer datasets. The additional edge count is presented as  multiples of the original graph's edge count.}
    \label{fig:edge count}
    \vspace{0pt}
\end{figure}
However, their effectiveness often comes at the cost of a substantial increase in the number of edges. As illustrated in Figure~\ref{fig:edge count}, we observe that: 1) GSL methods require a considerable number of additional edges to achieve meaningful performance gains --- indeed, some methods even exhibit adverse effects when the additional edge count is less than a certain multiple of the original graph (\eg 4 times for GRCN~\cite{yu2021graph} and RWGSL~\cite{shen2025towards} on the Cora dataset);  2) Optimal performance is typically attained only when the number of edges increases by an order of magnitude (\eg 36 and 28 times for GRCN~\cite{yu2021graph} and RWGSL~\cite{shen2025towards} on the Cora dataset). These findings highlight a critical trade-off in existing GSL approaches, which can significantly impede their practical applications. While GSL can indeed improve downstream task performance, the resultant expansion in graph size leads to significantly increased storage demands and computational inefficiency. This limitation motivates a key research question: \textit{How can we construct an informative yet concise graph structure that enhances downstream tasks while adding limited edges?}

To address this, we first identify that the limitations of GSL primarily lie in their purely similarity-based edge construction strategies. 
Such approaches typically connect each node to its most similar neighbors, leading to high mutual similarity among these neighbors and substantial aggregated information redundancy.
This redundancy is evident in Figure \ref{fig:neighbor_sim}, where the average mutual similarities among selected neighbors remain high even when selecting a considerable number of neighbors. This causes downstream GNNs to overemphasize redundant information from similar neighbors, diluting the influence of less similar but potentially more informative nodes,  and ultimately  hindering their effectiveness.  Thus, it is essential to reconsider the node selection strategy and select less redundant neighbors.

Towards this end, we propose a novel Informative Graph Structure Learning method (InGSL), which jointly considers both similarity and informational diversity when constructing the refined graph.
Figure \ref{fig:framework} illustrates how InGSL differs from existing embedding-based GSL methods.
Specifically, for each target node, InGSL employs a learnable selection procedure, guided by mutual information between the selected neighbors and all candidate neighbors to enhance the information richness of the selected neighbors. Notably, InGSL is simple, model-agnostic, and can be seamlessly integrated with various embedding-based GSL frameworks. In our experiments, we implement InGSL in six baselines, significantly improving their performance at a reduced scale of edges.

\begin{figure}
    \centering    \includegraphics[width=1.0\linewidth, trim = 0 0 0 -20]{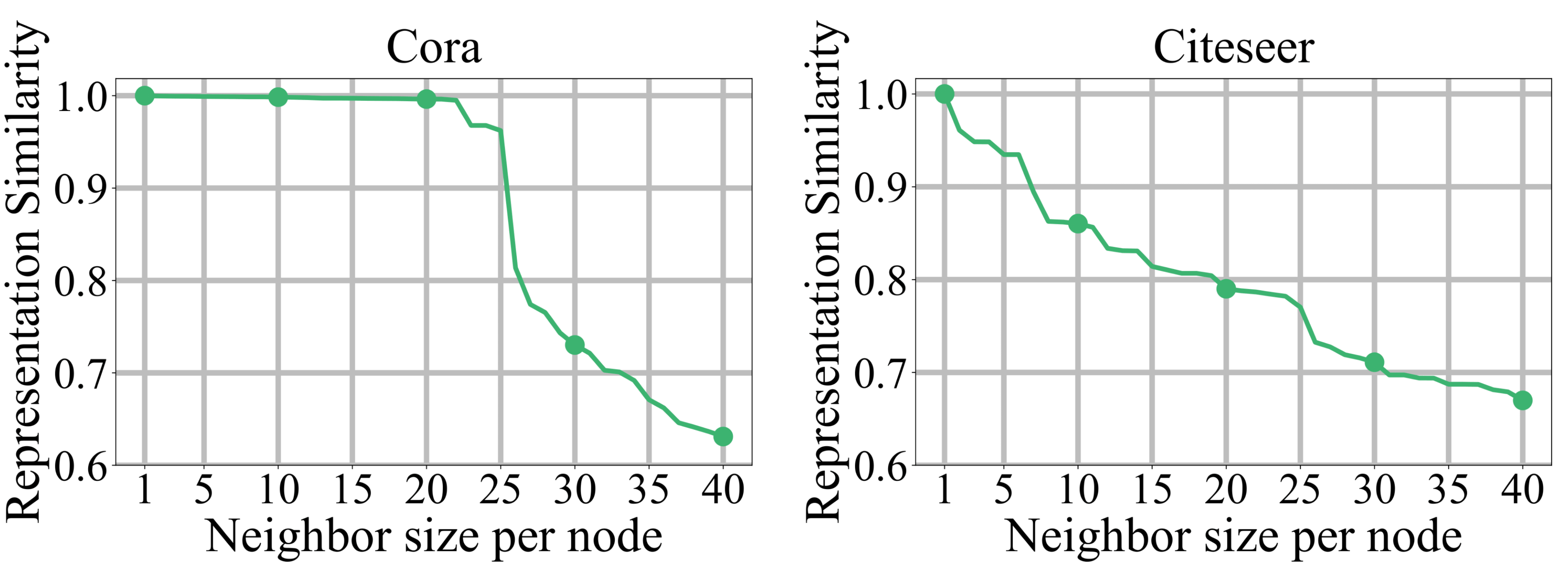}
    \setlength{\belowcaptionskip}{-10pt}
    \setlength{\abovecaptionskip}{0pt}
    \caption{Average pairwise similarity among selected neighbors versus neighbor size per node on the GRCN method~\cite{yu2021graph}.}
    \label{fig:neighbor_sim}
    \vspace{0pt}
\end{figure}

In summary, our main contributions are as follows:

\begin{itemize}
\item We reveal that the effectiveness of recent GSL methods often relies on introducing a large number of additional edges, and we attribute this limitation to their purely similarity-based neighbor selection, which tends to produce redundant edges.
\item We propose a novel informative graph structure learning strategy (InGSL) that selects both similar and diverse neighbors based on mutual information estimation, thereby reducing redundancy.
\item We conduct extensive experiments on six benchmark datasets, demonstrating the effectiveness of InGSL while introducing only a limited number of edges.
\end{itemize}


\section{Preliminaries}
\label{sec2}
\subsection{Notations} 
Suppose we have a graph $\mathcal{G}=(\mathcal{V}, \mathcal{E}, \mathbf{A}, \mathbf{X})$, where $\mathcal{V}$ is the set of $n$ nodes $\{v_1,...,v_n\}$ and $\mathcal{E}$ is the set of edges.
Let $\mathbf{A}$ be the initial adjacency matrix of the graph such that $\mathbf{A}_{ij}=1$ if an edge exists between node $v_i$ and $v_j$ and $\mathbf{A}_{ij}=0$ otherwise.
$\mathbf{X}=[x_1,...,x_n]\in \mathbb{R}^{n \times d}$ represents the node feature matrix, where each column $x_i$ corresponds to the feature vector of node $v_i$.
Let 
$\mathbf{D}$ denote the diagonal degree matrix defined as $\mathbf{D}_{ii}=1+\sum_{j}\mathbf{A}_{ij}$; 
and 
$\hat{\mathbf{A}}$ denotes the normalized adjacency matrix with self-loop, \eg $\hat{\mathbf{A}}=\mathbf{D}^{-\frac{1}{2}} (\mathbf{A}+\mathbf{I}) \mathbf{D}^{-\frac{1}{2}}$ for symmetric normalization. 

\begin{figure}[t]
\centering
\includegraphics[width=0.85\linewidth, trim = 80 20 80 -10]{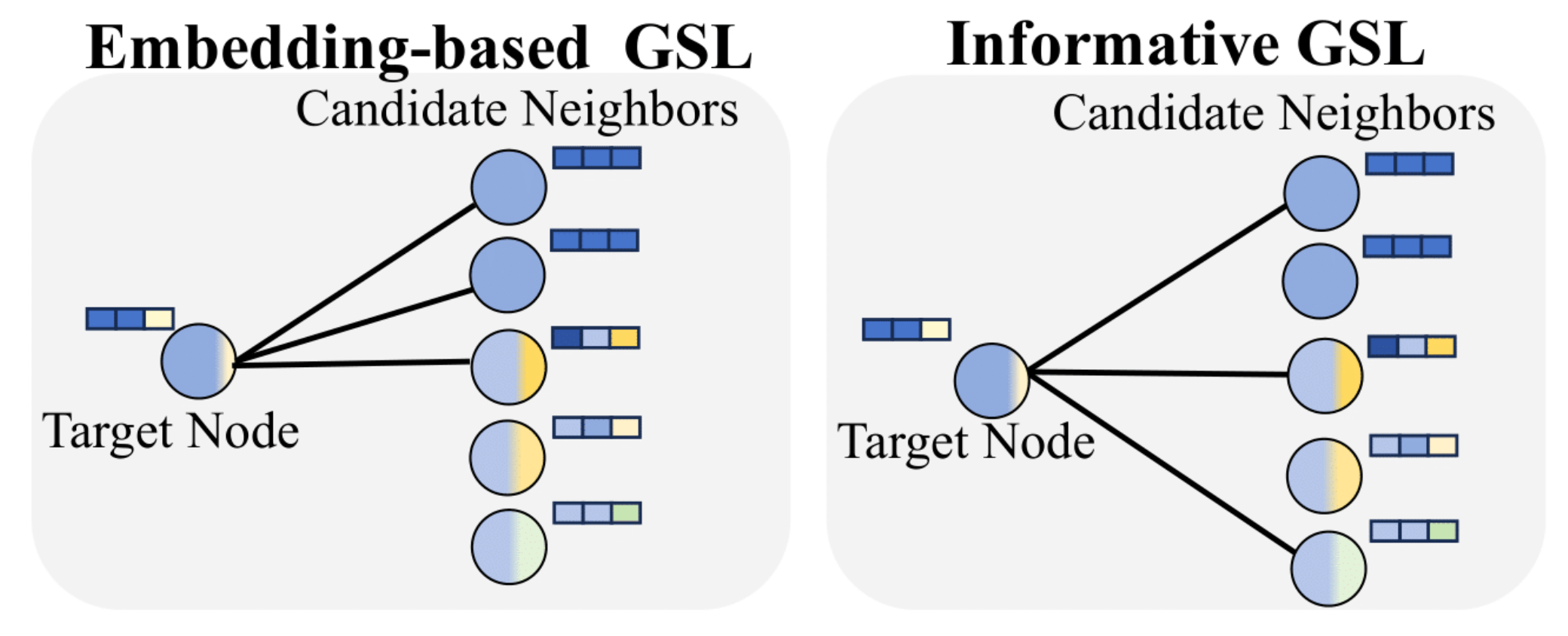}
\caption{Comparison of InGSL with existing embedding-based GSL methods. Existing embedding-based GSL methods rely primarily on embedding similarity for neighbor selection, while InGSL incorporates both similarity and information diversity.}
\label{fig:framework}
\vspace{-10pt}
\end{figure}

\subsection{Graph Neural Networks} 
Graph Neural Networks (GNNs) have emerged as a leading method for extracting knowledge from graph-structured data~\cite{kipf2017semisupervised,FengCLSZ22,GaiDZM025}.
For convenience, we use the representative Graph Convolutional Network~\cite{kipf2017semisupervised} (GCN) as an illustrative example. The operation at the $l$-th layer in GCN can be formulated as:
\begin{equation}
\mathbf{Z}^{(l)}=\text{ReLU}(\hat{\mathbf{A}}\mathbf{Z}^{(l-1)}\mathbf{W}^{(l)}),
\end{equation}
where $\mathbf{W}^{(l)}$ is a trainable weight matrix responsible for feature transformation. Remarkably, the normalized adjacency matrix $\hat{\mathbf{A}}$ constitutes a core component of GCNs, and its quality is crucial for model performance. 

\subsection{Graph Structure Learning}
Graph Structure Learning ~\cite{zhou2023opengsl} (GSL) aims to refine the adjacency matrix $\mathbf{A}$ and has demonstrated effectiveness in improving both the accuracy and robustness of GNNs~\cite{liu2022towards}. For clarity, we employ the notation $\mathbf{S}$ to represent the augmented adjacency matrix, where each entry $\mathbf{S}_{ij}$ represents the edge weight between nodes $v_{i}$ and $v_{j}$. We use $\mathbf{S}_{ij}=0$ to denote the absence of an edge. 

The adjacency matrices in GSL models are typically optimized under the supervision of downstream task performance. Taking the representative node classification task as an example, the objective function can be formulated as follows:
\begin{equation}
    \mathcal{L}_{GSL} =\mathcal{L}_{{Task}}(\text{GNN}_T(\mathbf{S},\mathbf{X}),\mathbf{Y})+\lambda\mathcal{L}_{{Reg}}(  
    \mathbf{S}),
    \label{eq_summary_gsl}
\end{equation}
where GNNs are commonly used to encode graph structures and node features into node representations and leverage the available labels $\mathbf{Y}$ as a supervised signal to jointly optimize the adjacency matrix $\mathbf{S}$ and the GNN encoder. Here, we use a subscript ($\text{GNN}_T$) to indicate that this GNN is employed for the downstream task.
The regularization term $\mathcal{L}_{{Reg}}$ is designed to enforce specific properties on the learned adjacency matrix $\mathbf{S}$, such as connectivity~\cite{zou2023se} and feature smoothness~\cite{yu2021graph}, $\lambda$ is a trade-off hyperparameter.  

Since directly optimizing $\mathbf{S}$ involves a large number of parameters, recent studies often employ an embedding-based paradigm~\cite{yu2021graph, in2024self}. Specifically, these approaches employ an additional auxiliary GNN encoder to extract node embeddings and subsequently generate the adjacency matrix based on embedding similarities:
\begin{equation}
\label{embed}
    \mathbf{E} = \text{GNN}_S(\mathbf{A},\mathbf{X}),
\end{equation}
\begin{equation}
\label{sim}
    \mathbf{S} =\text{Top}K(\mathbf{E \mathbf{E}^{T}}).
\end{equation}
where $\mathbf{{E}}$ represents a learnable node embedding matrix encoded by another GNN ($\text{GNN}_{S}$) tailored for graph construction. $\text{Top}K(\cdot)$ denotes the function that retains only the top-$K$ largest values for each row and sets others as $0$. The adjacency matrix is constructed based on the similarity between two node embeddings (\eg inner product as shown in Equation (\ref{sim}), which can also be replaced by cosine similarity~\cite{chen2020iterative,in2024self} or neural networks~\cite{liu2022compact}). Only the nodes with the top-$K$ highest similarity are connected to each node. This approach enhances graph homophily and yields better performance~\cite{zhou2023opengsl}.

\subsection{Analyses on Existing GSL Methods}
While existing embedding-based GSL methods have achieved great success, we contend that their effectiveness comes at the cost of a substantial increase in the number of edges. Specifically, we empirically investigate the performance of two representative GSL methods (GRCN ~\cite{yu2021graph} and RWGSL~\cite{shen2025towards}) on Cora and Citeseer datasets. As shown in Figure \ref{fig:edge count}, we observe that optimal performance is typically attained when adding over 20 times the original number of edges. Moreover, when imputing a limited number of edges (\eg 4 times the original), we find that GSL can have adverse effects. This edge-scale explosion significantly hinders the application of GSL. The increased edge scale not only significantly increases storage demands --- particularly considering that significantly denser graphs are highly difficult for distributed storage~\cite{disgraph}; but also incurs additional computational overhead in downstream graph analyses, given that the complexity of GNNs is often linear or super-linear with respect to the edge scale~\cite{GraphConden1,GraphConden2}. Therefore, it is essential to explore new strategies for constructing informative yet concise graph structures.

To address this challenge, we attribute the limitations of existing GSL to their conventional neighbor selection strategies. Blindly selecting the most similar neighbors may indeed increase graph homophily but can also increase redundancy in the graph structure. As shown in Figure \ref{fig:neighbor_sim}, highly similar neighbors often share mutually redundant contents, thereby contributing little additional useful information. We also demonstrate this by the following Lemma:

\begin{lemma}
\label{lemma1}

Consider a GSL method that selects neighbors based on cosine similarity of embeddings. Let N denote the number of additional neighbors whose similarity with the target node exceeds $\epsilon$. The average pairwise embedding similarity among these neighbors is bounded by:
\begin{equation}
    \overline{s} \geq \frac{N\epsilon^{2}-1}{N-1}.
\end{equation}

\label{lemma 1}
\end{lemma}

This lemma characterizes the redundancy effect that arises from similarity‑based neighbor selection in GSL. Specifically, increasing the threshold $\epsilon$ (\ie selecting only the most similar nodes as neighbors)  raises the lower bound on the neighbors' average pairwise similarity. Consequently, the information aggregated from these neighbors becomes increasingly redundant.
The proof is presented in Appendix \ref{proofLemma1} .

We further study how the redundancy affects the effectiveness of GNNs. The following lemma formalizes the impact in a classification setting.

\begin{lemma}
\label{lemma2}
Consider a node classification problem on a graph with GNNs. Let $\mathbf{Z}_{i}^{(l)}$ denote the $l$-layer GNN representation of node $v_i$, which is subsequently passed to a linear classifier with the parameters $\textbf{W}_{c}$. Let $\mathbf{Y}_{i}$ be the one‑hot label vector of $v_{i}$, and $\|\mathbf{Z}^{(l)}_{i}\|_{2} \leq B, \forall v_{i} \in \mathcal{V}$.  Classification accuracy is quantified using the cross‑entropy loss $\mathcal{L}(\mathbf{Z}_{i}^{(l)}\mathbf{W}_{c} , \mathbf{Y}_{i}) = -\text{log}(\mathbf{Y}_{i}\cdot \text{softmax}(\mathbf{Z}_{i}^{(l)}\mathbf{W}_{c}))$. Now consider the case where the neighbors are highly similar, specifically when the cosine similarity satisfies $\cos(\mathbf{Z}_{i}^{(l)}, \mathbf{Z}_{j}^{(l)}) \geq \epsilon, \forall v_j \in \mathcal{N}(v_i)$. Under additional GNN aggregation $\mathbf{Z}_{i}^{(l+1)} = \sum_{v_j \in \mathcal{N}(v_i)} \mathbf{S}_{ij} \mathbf{Z}_{j}^{(l)}$, the classification‑accuracy gain is bounded by: 
\begin{equation}
    |\mathcal{L}(\mathbf{Z}_{i}^{(l+1)}\mathbf{W}_{c}, \mathbf{Y}_{i}) - \mathcal{L}(\mathbf{Z}_{i}^{(l)}\mathbf{W}_{c}, \mathbf{Y}_{i})| 
\leq 2 B \lVert \mathbf{W}_{c}\rVert_{2}  \sqrt{1 - \epsilon}.
\end{equation}

\label{lemma 2}
\end{lemma}

The proof is presented in Appendix \ref{appx:lemma2}. The lemma indicates that highly similar neighbors hinder the effectiveness of GNNs---larger $\epsilon$ may lead to a smaller accuracy gain for GNNs.
Excessive aggregation of redundant information offers limited additional utility.
Empirical evidence from Figure \ref{fig:edge count} shows that GSL achieves significant performance gains only when incorporating substantial neighbors with lower similarity.
Both our theoretical and empirical analysis reveal the same limitation—relying solely on similarity-based neighbor selection increases structural redundancy and forces GSL methods to grow edge counts excessively to achieve meaningful performance gains.
The excessive focus on the most similar neighbors may have adverse effects, potentially causing pattern collapse and neglect of other informative signals. Given these analyses, it is essential to reconsider and refine existing similarity‑based neighbor selection strategies.

\section{METHODOLOGY}
\label{sec3}
Given the importance of constructing an informative graph structure that enhances downstream tasks while adding limited edges, 
we propose the informative graph structure learning (InGSL) strategy.
Going beyond similarity, InGSL introduces information diversity to guide neighbor selection.

InGSL employs a learnable neighbor-selection mechanism, guided by the mutual information between the selected and all candidate neighbors to maximize the informational richness of the selected subset. Designed as a plug-in module, InGSL can be integrated into various existing GSL methods, reducing graph edge density while preserving or even improving model performance. Concretely, InGSL primarily modifies the neighbor-selection process of existing GSL approaches. It initially follows the conventional GSL pipeline to generate candidate neighbors with the adjacency matrix $\mathbf{S}$ (which typically contains a relatively large number of neighbors to ensure effectiveness), and then refines and compacts the adjacency matrix using the following formulation:

\begin{equation}
\label{prune}
    \tilde {\mathbf{S}}_{ij} = \psi(\mathbf{S}_{ij} \cdot {w}_{ij}),
\end{equation}
where 
\begin{equation}
\label{mask}
    \psi(x) = \left\{
    \begin{array}{lr}
        \frac{1}{1+e^{-x}},  &{x \geq \epsilon,}\\
        0, &{x < \epsilon,}
    \end{array} 
    \right.
\end{equation}
Here, $w_{ij}$ denotes the learnable diversity score associated with candidate neighbor $v_j$,  which is combined with the original similarity $\mathbf{S}_{ij}$ to guide the neighbor selection process. The function $\psi(\cdot)$ performs selection and normalization, retaining important edges while normalizing their weights. The parameter $\epsilon$ serves as a threshold to filter out edges whose adjusted weight $\tilde {\mathbf{S}}_{ij}$ falls below $\epsilon$.  This threshold can be set according to the desired number of neighbors to be retained.

Estimating the diversity score $w_{ij}$ is non-trivial, and it is difficult to specify a suitable manual definition. Therefore, we adopt an adaptive learnable strategy. A direct parameterization of each \(w_{ij}\) as an independent learnable weight would result in a large parameter space, leading to optimization challenges and heavy computational overhead. To mitigate this, we re-parameterize \(w_{ij}\) through a learnable function:

\begin{equation}
\label{prob}
    {w}_{ij} = f(\mathbf{E}_{i},\mathbf{E}_{j}),
\end{equation}
where the diversity score \(w_{ij}\) is computed from the node embeddings \(\mathbf{E}_{i}\) and \(\mathbf{E}_{j}\) generated by \(\text{GNN}_S\).
These embeddings capture the characteristics of the two endpoint nodes and thus naturally provide signals regarding the informativeness of the corresponding edge.
The function \(f(\cdot, \cdot)\) is learnable, and can be implemented using various neural architectures, such as a bilinear layer or MLP. Empirically, we observe that the bilinear formulation yields the best results (see Appendix~\ref{appex:netfunc1} for implementation details). 

To ensure that InGSL captures informative yet compact structural information, we develop a mutual-information-based objective to guide the learning of the diversity score \(w_{ij}\). Specifically, we aim for the refined and compact graph structure to preserve information richness comparable to that of the original candidate set with substantial edges. The objective is:
\begin{equation}
\label{Eq:MIloss}
    \mathcal{L}_{MI} = -\hat{I}(\tilde{\mathbf{Z}}; \mathbf{Z}) = -\frac{1}{n}\sum_{i=1}^{n} \log \frac{ e^{\text{sim}(\tilde{\mathbf{Z}}_{i}, \mathbf{Z}_{i})} }{ \sum_{v_j \in \mathcal{B}} e^{\text{sim}(\tilde{\mathbf{Z}}_{i}, \mathbf{Z}_{j})} },
\end{equation}
\begin{equation}
\label{Eq:representation}
    \tilde{\mathbf{Z}} = \text{GNN}_T(\tilde{\mathbf{S}}, \mathbf{X}), \quad \mathbf{Z} = \text{GNN}_T(\mathbf{S}, \mathbf{X}),
\end{equation}
where \(\tilde{\mathbf{Z}}\) and \(\mathbf{Z}\) denote the node representations learned from the compact graph and the original constructed graph, respectively. Maximizing the mutual information between them ($\hat{I}(\tilde{\mathbf{Z}}; \mathbf{Z})$) encourages the compact graph structure to retain the original’s rich information. This objective naturally guides \(w_{ij}\) to capture informative signals while avoiding redundancy in edge selection. We approximate mutual information using a softmax loss~\cite{LiangDMZMS23}, enabling efficient implementation and optimization, where \(\mathcal{B}\) denotes the sampled node set for computation (see Appendix ~\ref{appx:MIloss} for further details).  The overall training objective is:
\begin{equation}
\label{Eq:InGSLAlllos}
    \mathcal{L} = \mathcal{L}_{GSL} + \beta \mathcal{L}_{MI},
\end{equation}
where \(\beta\) controls the relative contribution of the mutual-information term, $\mathcal{L}_{GSL}$ is the objective function of base GSL methods ( Equation (\ref{eq_summary_gsl})).

As shown, InGSL follows the same optimization procedure as the base GSL framework, with the primary distinction that neighbors are selected according to Equation (7) rather than Equation (4), thereby incorporating diversity into the selection process. In addition, a mutual‑information‑based objective is introduced to guide neighbor selection. This lightweight modification allows InGSL to be seamlessly integrated into existing GSL architectures, producing graph structures that are both compact and  informative. Moreover, the additional operations introduce minimal computational overhead. Detailed complexity analysis and runtime comparisons are provided in Appendix~\ref{appex:complex} and Appendix \ref{appex:effi}.
\begin{table}[t]
\begin{minipage}[t]{\linewidth}
\renewcommand{\arraystretch}{1.1} 
\centering
\setlength{\abovecaptionskip}{0.cm}
\caption{Detailed statistics of node classification datasets.}

\label{Tab:dataset stastics}
\setlength{\tabcolsep}{0.008\textwidth}{
{
\begin{tabular}{lcccccc}
\hline
        Dataset & \#Nodes & \#Edges & \#Feat. & \#Avg.degree  & \#Homophily \\ \hline
      
        Cora & 2,708 & 5,278 & 1,433 & 3.9 &  0.81 \\ 
        
        Citeseer & {3,327} & {4,552} & {3,703} &  {2.7} &   0.74  \\ 

        Pubmed & 19,717 & 44,324 & 500 & 4.5 &  0.80  \\
         
        Ogbn-arxiv & 169,343  & 1,157,799  &  767 & 13.67 & 0.65 \\
        BlogCatalog & 5,196 & 
        171,743 & 
        8,189 & 66.1 & 
         0.40  \\
        Roman-empire & 22,662 & 32,927 & 300 & 2.9 &  0.05  \\ 
        \hline
\end{tabular}}}
\end{minipage}%
\vspace{-10pt}
\end{table}
\begin{table*}[ht]
\renewcommand{\arraystretch}{1.2}
	\centering
    \setlength{\abovecaptionskip}{0.cm}
\caption{Node classification accuracy±std comparison(\%). 
Each experiment is repeated 5 times with different random seeds.
"OOM" denotes out of memory.
The top-performing results are marked in \textbf{bold}.}
	\label{Tab:performance}
    	\resizebox{\linewidth}{!}
{
\begin{tabular}[t!]{c|cc|cc|cc|cc|cc}
        \toprule
        {Dataset} &  \multicolumn{2}{c|}{Cora} &  \multicolumn{2}{c|}{Citeseer} & \multicolumn{2}{c|}{Pubmed}  &  \multicolumn{2}{c|}{Blogcatalog} & \multicolumn{2}{c}{Ogbn-arxiv} \\
        \midrule
        {Edge Reduction $(\%)$} & 30$\%$ & 
        50$\%$ & 30$\%$ & 50$\%$  & 30$\%$ & 50$\%$ & 30$\%$ & 50$\%$ & 30$\%$ & 50$\%$\\
        \midrule
GRCN    & 83.73±0.12 & 83.53±0.24 &{72.41±0.55} & {71.93±0.41} &{77.93±0.29} & {74.57±1.03} &{73.43±0.54} & {73.29±0.30} &OOM &OOM \\
GRCN+InGSL & \textbf{85.20±0.50} & \textbf{85.25±0.30} &\textbf{72.58±1.15} & \textbf{72.75±0.95} &\textbf{79.44±0.88} & \textbf{79.20±0.70} &\textbf{76.63±0.18} & \textbf{76.22±0.30} &{OOM} & {OOM} \\
\midrule
IDGL & 84.20±0.17 & 84.03±0.30 &{72.47±0.85} & {72.13±0.51} &{82.87±0.12} & {82.53±0.32} &{89.05±0.06} & {88.81±0.27} &{69.84±0.32} & {70.14±0.09} \\
IDGL+InGSL & \textbf{84.41±0.43} & \textbf{84.62±0.35} & \textbf{73.42±0.81} & \textbf{73.20±0.85} &\textbf{83.00±0.72} & \textbf{82.85±0.37} &\textbf{91.32±0.81} & \textbf{91.20±0.90} &\textbf{70.69±0.08} & \textbf{70.54±0.08} \\
\midrule
SUBLIME    & 82.35±0.64 & 81.55±1.34 &{71.70 ± 0.85} & {72.60±0.85} &{80.20±1.13} & {79.05±1.34} &{94.78±0.60} & {94.52±0.64} & 70.24±0.34 &69.59±0.57 \\
SUBLIME+InGSL & \textbf{83.42±0.96} & \textbf{82.95±0.56} &\textbf{73.05±0.33} & \textbf{73.36±0.49} &\textbf{81.21±0.15} & \textbf{81.50±0.30} &\textbf{95.79±0.16} & \textbf{95.48±0.28} &\textbf{71.33±0.11} & \textbf{71.23±0.31} \\
\midrule
RWGSL &{83.66±0.15} &{83.53±0.19} & {72.67±0.21} & {72.16±0.29} &{80.19±0.09} & {78.95±0.78} &{83.23±0.19} & {81.01±0.85} &{70.90±0.12} & {70.06±0.13} \\
RWGSL+InGSL & \textbf{85.47±0.32} & \textbf{84.73±0.40} & \textbf{73.67±0.81} & \textbf{73.50±0.85} &\textbf{82.43±0.60} & \textbf{82.37±0.90} &\textbf{92.39±0.21} & \textbf{91.99±0.69} &\textbf{71.72±0.21} & \textbf{70.50±0.45} \\
\midrule
UnGSL & 85.40±0.38 & 84.74±0.32 &{73.30±0.17} & {73.46±0.78} &{78.80±0.36} & {77.30±0.52} &{72.69±0.16} & {72.94±0.23} &{OOM} & {OOM} \\
UnGSL+InGSL & \textbf{85.83±0.14} & \textbf{85.17±0.71} & \textbf{73.92±0.06} & \textbf{74.35±0.40} &\textbf{80.10±0.87} &\textbf{80.37±0.57} &\textbf{73.60±0.10} &\textbf{73.58±0.43} &{OOM} & {OOM} \\
\midrule
PROSE & 81.30±0.43 & 81.57±0.24 &{72.00±0.29} & {71.60±0.22} &{82.55±0.45} & {82.35±0.05} &{75.57±0.30} & {75.44±0.54} &{68.44±0.05} & {68.40±0.15} \\
PROSE+InGSL & \textbf{81.77±0.49} & \textbf{81.73±0.39} & \textbf{72.83±0.26} & \textbf{73.33±0.33} &\textbf{82.65±0.45} &\textbf{82.50±0.40} &\textbf{76.65±0.54} &\textbf{76.48±0.26} &\textbf{71.05±0.19} & \textbf{70.71±0.10} \\
\bottomrule
\end{tabular}
} 
\end{table*}

\begin{figure*}[t]
    \centering    \includegraphics[width=0.90\linewidth, trim = 53 10 30 0]{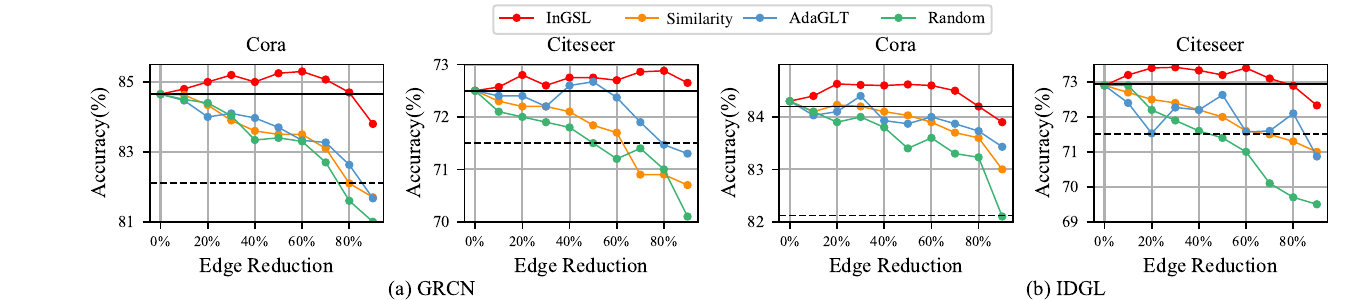}
    \caption{Performance of different edge reduction methods on GRCN and IDGL models across varying reduction levels.
    The solid line represents the accuracy of the graph structure generated by the vanilla GSL model with an optimal number of neighbors,
    and the dashed line shows the accuracy of GCN on the original graph data. }
    \label{sratio}
    \vspace{-10pt}
\end{figure*}
\section{EXPERIMENTS}
\label{sec4}
In this section, we conduct experiments to evaluate the effectiveness of the proposed InGSL strategy.
Our experiments aim to answer the following research questions:
\begin{itemize}[leftmargin=1.0em]
    \item \textbf{RQ1:} Can InGSL effectively reduce the number of edges in existing GSL methods while retaining or even improving their performance?
    \item \textbf{RQ2:} What is the impact of key configurations (\eg diversity score $w_{ij}$ , mutual-information-based objective $\mathcal{L}_{MI}$ and learnable function $f(\cdot)$) on the performance of InGSL? 
    \item \textbf{RQ3:} How well does the learned graph generated by InGSL generalize on different GNN backbones?
    \item  \textbf{RQ4:} Can InGSL enhance the robustness against random noise and graph topology attack?
\end{itemize}

\subsection{Experiment Settings}
\subsubsection{Datasets.}
To comprehensively evaluate the performance of InGSL, we closely follow previous work \cite{zhou2023opengsl} and select six widely used datasets on the node classification task, including 4 homophilous datasets ~\cite{sen2008collective,OGB} (Cora, Citeseer, Pubmed and Ogbn-arxiv) and 2 heterophilous datasets ~\cite{huang2017label,platonov2023critical} (Blogcatalog and Roman-empire). 
These datasets span a variety of homophily levels and graph sizes, allowing us to assess the effectiveness of InGSL under diverse conditions. 
For fair comparison, we strictly adhere to the data splits specified by the recently proposed GSL benchmark~\cite{zhou2023opengsl}. 
Detailed information about these datasets can be found in Table \ref{Tab:dataset stastics}.
Notably, due to the limited space, the results on the Roman-empire dataset are reported in the Appendix \ref{appex:roman}.
\subsubsection{Baselines.}
To demonstrate the versatility of InGSL across different GSL methods, we implement InGSL into six state-of-the-art GSL algorithms corresponding to the chosen datasets.
These GSL baselines include supervised methods (GRCN~\cite{yu2021graph}, IDGL~\cite{chen2020iterative}, PROSE~\cite{wang2023prose}, RWGSL~\cite{shen2025towards}), an unsupervised methods SUBLIME~\cite{liu2022towards}, and another GSL plug-in module UnGSL~\cite{han2025uncertainty} incorporated with the GRCN model which performs best as reported.
For all models, we report the average performance and standard deviation over five runs with different random seeds. 
Notably, although a few methods are not embedding-similarity-based~\cite{wang2021graph,zou2023se}, their performance is consistently inferior to that of similarity-based approaches~\cite{zhou2023opengsl} such as GRCN~\cite{yu2021graph} and IDGL~\cite{chen2020iterative}.
Thus, they are excluded from our comparative experiments.
The details of these baselines are as follows:
\begin{itemize}[leftmargin=1.0em]
    \item {GRCN} (PKDD'20) \cite{yu2021graph} first  extracts topological features and compute the similarity between nodes as the edge weights of the learned graph.
    \item {IDGL} (NIPS'20) \cite{chen2020iterative} iteratively learns both the graph structure and the node embeddings.
    \item {SUBLIME} (WWW'22) \cite{liu2022towards}  employs contrastive learning between the learned graph and an augmented graph to enhance the robustness of the graph structure.
    \item {PROSE} (KDD'23) \cite{wang2023prose} identifies influential nodes and reconstructs the graph structure by connecting these influential nodes.
    \item {RWGSL} (ICDE'25) \cite{shen2025towards} constructs the graph structure in a parameter-free manner by leveraging multiple metrics.
    \item {UnGSL} (WWW'25) \cite{han2025uncertainty} first estimates node uncertainty, then identifies and reinforces connections with high‑confidence neighbors using learnable thresholds.
\end{itemize}

\subsubsection{Configuration.}
\label{sec:conf}
For our method, hyperparameter
$\beta$ can be tuned within the range of [0,1], and the size of sampled node set in Equation (\ref{Eq:MIloss}) is fixed at 2,000. 
The hidden dimension of node embeddings is fixed at 128. 
We optimized InGSL using Adam optimizer, with the learning rate selected from the range [1e-5, 5e-2].
For the learnable function $f(\cdot)$, we use a bilinear function, as it generally achieves better overall performance. Details of the learnable function are provided in Appendix \ref{appex:netfunc1}.
Threshold $\epsilon$ is determined by both the edge reduction level and the edge weights during the training process. 
{For all baselines, we strictly adhere to their original settings for hyperparameter tuning to ensure that they attain the best performance. 
All GSL methods are evaluated based on the performance of GNNs on downstream tasks when using the learned structure.

\subsection{Main Results (RQ1)}
\subsubsection{Performance Comparison at Certain Levels.}
Table \ref{Tab:performance} presents how InGSL improves existing GSL methods with reduced edges. For a fair comparison, we select the optimal number of additional edges at which existing GSL methods achieve their best performance, and then evaluate the model performance across different levels of edge reduction of 30\% and 50\%. 
Here, we compare the performance of InGSL with that of the original GSL methods.
The results for more edge reduction levels and for the Roman‑empire dataset are provided in the Appendix \ref{appex:addresult}. 

We can observe that: 1) graphs constructed by InGSL consistently achieve significantly higher node classification accuracy compared to those generated by the base GSL model with similarity-based neighbor selection. 
2) In some cases, increasing the edge reduction level can further enhance the quality of the graph structure constructed by InGSL. 
For example, when the edge reduction level of GRCN + InGSL in the Citeseer dataset increases from 30$\%$ to 50$\%$, the accuracy also increases.
This is likely because base GSL models generate graph structures with many redundant edges.  
By connecting nodes to similar and diverse neighbors, InGSL can alleviate the dilution of informative signals from dissimilar but important neighbors.
In general, these results demonstrate InGSL’s effectiveness in selecting both similar and diverse neighbors to enhance the performance of graphs with limited edges.

\subsubsection{Performance Comparison with Varying Edge Scale.}
To further demonstrate the effectiveness of the proposed method, we compare InGSL with the following baselines: 1) Similarity-based neighbor selection; 2) AdaGLT~\cite{AdaGLT24}. To further demonstrate the effectiveness of InGSL in edge reduction, we also compare it with the representative graph sparsification module AdaGLT~\cite{AdaGLT24}, which adaptively filters out low-similarity edges for each node; 3) Random neighbor selection. We test the performance of these methods on GRCN and IDGL with varying levels of edge reduction, as reported in Figure \ref{sratio}.
We can observe that 1) InGSL consistently outperforms AdaGLT, similarity-based neighbor selection, and random reduction methods across all edge reduction levels. 
As the reduction level increases, the overall relative improvement of InGSL over the baseline also rises.
2) As the edge reduction level increases, InGSL  can surpass the performance of the vanilla GSL model. 
For example, on the Citeseer dataset, the graph structure generated by GRCN+InGSL at the 80$\%$ level significantly outperforms the base GSL model with an optimal number of neighbors.
Furthermore, the graph structure generated by IDGL+InGSL on the Cora dataset at the 70$\%$ level still achieves accuracy comparable to the vanilla GSL model.
The results demonstrate  InGSL’s effectiveness in building informative graph structures and achieving superior performance with a limited number of edges.
\begin{table}[!t]
	\centering
        \small
        \setlength{\abovecaptionskip}{0.1cm}
	\caption{Ablation study on the InGSL when integrating with SUBLIME model~\cite{liu2022towards}.}
\setlength{\tabcolsep}{0.008\textwidth}
	
{
\begin{tabular}[t!]{c|cc|cc}
        \toprule
        {Dataset} &  \multicolumn{2}{c|}{Citeseer}  &  \multicolumn{2}{c}{Blogcatlog}\\
        \midrule
        {Edge Reduction $\%$} & 30$\%$ & 
        50$\%$  & 
        30$\%$ & 
        50$\%$  
        \\
        \midrule
w/o $w_{ij} $    & 72.65±0.11 & 72.82±0.96 &  {95.01±0.24} & {94.92±0.44} \\
w/o $\mathcal{L}_{MI}$  & 72.55±0.05 & {72.70±0.51}  & {94.90±0.09} & {94.75±0.11}\\ 
InGSL  & \textbf{73.05±0.33} & \textbf{73.36±0.49}  & \textbf{95.79±0.16} & \textbf{95.48±0.28}\\ 
\bottomrule
\end{tabular}
}
\label{tab:ab}
\vspace{-10pt}
\end{table}
\begin{table}[!t]
	\centering
        \small
        \setlength{\abovecaptionskip}{0.1cm}
        \renewcommand\arraystretch{1.2}
	\caption{Performance comparison of the bilinear function and MLP used as the learnable function in Equation (\ref{prob}).}
\setlength{\tabcolsep}{0.008\textwidth}
	
{
\begin{tabular}[t!]{c|cccc}
        \toprule
        {Edge Reduction $\%$} & 30$\%$ & 
        50$\%$  & 
        70$\%$ & 
        90$\%$  
        \\
        \midrule
Similarity-based    & 83.73±0.12 & 83.53±0.24 &  {83.10±0.20} & {81.73±0.58} \\
MLP  & 84.75±0.07 & {84.82±0.28}  & {84.51±0.99} & {83.15±1.20}\\ 
Bilinear Function  & \textbf{85.20±0.50} & \textbf{85.25±0.30}  & \textbf{85.07±0.32} & \textbf{83.67±0.91}\\ 
\bottomrule
\end{tabular}
}
\label{tab:sigma}
\vspace{-9pt}
\end{table}

\subsection{Ablation Studies (RQ2)}
\subsubsection{Analysis on Diversity Score $w_{ij}$.}
To evaluate the effectiveness of the estimated information diversity in Equation (\ref{prune}), we remove the estimated diversity score $w_{ij}$ in Equation (\ref{prune}) and examine the corresponding results. As shown in Table \ref{tab:ab}, removing $w_{ij}$ leads to a performance drop. 
These results demonstrate that accounting for neighbor diversity in graph structure learning improves the accuracy of GNNs when the number of edges is limited.

\subsubsection{Analysis on mutual information term $\mathcal{L}_{MI}$.}
To assess the effectiveness of the mutual information term in Equation (\ref{Eq:InGSLAlllos}), we remove $\mathcal{L}_{MI}$ by setting $\beta=0$. 
As shown in Table \ref{tab:ab}, removing $\mathcal{L}_{MI}$ results in degraded model performance.
We further evaluate the GRCN+InGSL model by varying $\beta$ within the interval [0,1].
As shown in Figure \ref{fig:alphabeta analysis}, a positive $\beta$ consistently outperforms $\beta=0$.
These results demonstrate  that maximizing the mutual information between selected neighbors and all candidates effectively retains edges carrying important semantic information, thereby improving the performance of the constructed informative graph.

\subsubsection{Effects of Different Learnable Functions $f(\cdot)$.}
\label{netfunc}
We evaluate the effect of different learnable functions by testing the GRCN~\cite{yu2021graph} model on the Cora dataset, using either an MLP or a bilinear function as $f(\cdot)$ in Equation (\ref{prob}). As shown in Table \ref{tab:sigma}, "Similarity-based" represents the baseline where GRCN reduces edges through similarity-based neighbor selection. We observe that:
1) The bilinear function outperforms the MLP across different edge reduction levels. This can be attributed to the bilinear function’s ability to directly measure pairwise node redundancy from their embeddings. Such redundancy can then be transformed into information diversity through the learnable mutual information optimization strategy, enabling better capture of diverse information.
2) Both the MLP and the bilinear function as learnable functions deliver performance gains over the similarity-based neighbor selection, indicating that incorporating learnable neighbor selection guided by mutual information in graph structure learning can effectively eliminate redundant structures and boost GNN performance.
\begin{figure}
    \vspace{0pt}
    \centering    \includegraphics[width=0.850\linewidth, trim = 45 20 20 10]{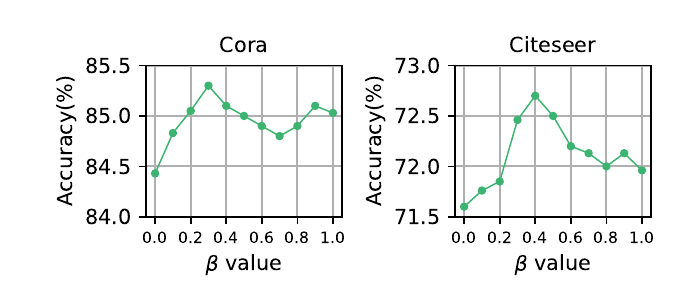}
    \caption{Comparsion of different $\beta$ on Cora and Citeseer datasets with edge reduction level of 50$\%$.}
    \label{fig:alphabeta analysis}
    \vspace{-10pt}
\end{figure}
\subsection{Generalizability Analysis (RQ3)}
To assess the generalizability of graph structures constructed by InGSL across different GNN backbones, we first pretrain the GSL+InGSL to generate informative graph structures with 50$\%$ edge reduction level. We then train different GNN backbones from scratch, using the learned informative structures as input.
We also compare the results with the graph structure generated by the vanilla GSL model with an optimal number of
neighbors (\ie without edge reduction).
We use GPRGNN~\cite{chien2020adaptive} and APPNP~\cite{gasteiger2018combining} as the GNN backbones and present a comparison of FLOPs (Floating Point Operations Per Second) and accuracy in Table \ref{tab:general}.
We observe that: 1) Informative graph structures generated by InGSL consistently outperform the original learned structures across a variety of GNN backbones. For example, on the Cora dataset with the GPRGNN backbone, a graph produced by InGSL with a 50$\%$ edge reduction achieves an average accuracy improvement of 0.6$\%$ over the structure obtained from the vanilla GSL model. These results indicate that the benefits of selecting informative neighbors generalize effectively to different GNN backbones.
2) The informative graph structures generated by InGSL can effectively reduce the FLOPs of different GNN backbones. For instance, on the Cora dataset with the APPNP model, a graph with a 50$\%$ edge reduction achieves a 34$\%$ decrease in FLOPs. This improvement stems from the neighbor‑aggregation paradigm commonly used in GNNs, where removing redundant edges reduces the number of aggregation operations and consequently lowers the overall computational cost.
Overall, the results demonstrate that the graph structures produced by InGSL exhibit strong generalization across a variety of GNN backbones, consistently delivering higher accuracy while reducing computational overhead. 
\begin{table*}[!t]
	\centering
        \tiny
        \setlength{\abovecaptionskip}{0.2cm}
	\caption{Generalizability of InGSL with different GNN backbones on Cora and Citeseer datasets. 
 The value in bold signifies the top-performing result.}
	\label{tab:general}
	\resizebox{\linewidth}{!}
{
\begin{tabular}[t!]{c|cccc|cccc}
        \toprule
        {Methods} &  \multicolumn{4}{c|}{GPRGNN} &  \multicolumn{4}{c}{APPNP} \\
        \midrule
        {Datasets}&  \multicolumn{2}{c}{Cora} &  \multicolumn{2}{c|}{Citeseer} &  \multicolumn{2}{c}{Cora} & \multicolumn{2}{c}{Citeseer}\\
        \midrule
        {} & FLOPs(M) & 
        Accuracy($\%$) & FLOPs(M) & Accuracy($\%$) & FLOPs(M) & 
        Accuracy($\%$) & FLOPs(M) & Accuracy($\%$)\\
        \midrule
GRCN    & 115 & 84.77±0.67 &{251} & {73.40±0.80} & {112} & {84.53±0.21} & {202} & {72.83±0.42} \\
GRCN+InGSL & {76} & \textbf{84.93±0.25} &{191} & \textbf{73.80±1.28} & {71} & \textbf{84.80±0.44} & {175} & \textbf{74.57±0.60}\\
\midrule
IDGL & 184 & 84.50±0.09 &{260} & {73.17±0.51} & {117} & {85.00±0.44} & {205} & {73.07±0.25} \\
IDGL+InGSL & {105} & \textbf{85.37±0.42} & {196} & \textbf{74.70±0.17} & {80} & \textbf{85.61±0.36} & {180} & \textbf{74.03±0.21} \\
\bottomrule
\end{tabular}
}
\end{table*}
\subsection{Robustness Analysis (RQ4)}

\subsubsection{Robustness Analysis \wrt Random Noise.}

To evaluate the robustness of informative graph structures generated by GSL+InGSL under structural and feature noise, we compare the accuracy of GSL models using graphs with a 50$\%$ edge reduction level generated by GRCN with similarity-based neighbor selection (\ie similarity-based edge reduction) and GRCN+InGSL across various noise levels. We also compare the results with the performance of the vanilla GRCN model without edge reduction (\ie w/o edge reduction).
For structural noise, we randomly add or remove edges from the original graph structure. For feature noise, we randomly mask the initial node features. 
Experiments are conducted on the Cora and Citeseer datasets, with results shown in Figure \ref{fig:robust1}.
The results for random feature noise are provided in Appendix \ref{appex:featnoise}.
Our observations are as follows: 1) InGSL consistently improves the performance of graph structures generated by the baseline model across different noise levels; 2) The relative improvement of InGSL increases as the noise level rises. For example, in edge deletion experiments on the Cora dataset, InGSL surpasses similarity-based neighbor selection and the vanilla GSL model by a larger margin at higher noise levels.
In summary, these results indicate that InGSL enhances robustness against both structural and feature noise.
\subsubsection{Robustness Analysis \wrt Graph Topology Attack.}
To assess the robustness of InGSL against graph topology attacks, we use the Cora and Citeseer datasets and evaluate graph structures generated by GRCN with similarity-based neighbor selection and GRCN+InGSL under a 50$\%$ edge reduction level. We further compare the results with the performance of the vanilla GRCN model (\ie without edge reduction) .
Specifically, we first select the largest connected component of the graph structure and employ MetaAttack~\cite{sun2020adversarial}, a representative non-targeted adversarial topology attack method, to generate perturbed graphs.
The perturbation rate, \ie the ratio of modified edges is varied from 0$\%$ to 25$\%$ with a step of 5$\%$.
The results are presented in Figure \ref{fig:robust3}.
Our observations are as follows: 
1) InGSL consistently outperforms both the vanilla GSL model and similarity-based neighbor selection across all perturbation ratios.
2) The relative improvement of InGSL increases as the perturbation ratio rises.
Furthermore, the improvement of InGSL over similarity-based neighbor selection is greater than that of the vanilla GSL model over similarity-based selection.
These results demonstrate the superiority of InGSL in generating refined and compact graph structures which, with a limited number of edges, are capable of significantly enhancing the robustness of base GSL models against graph topology attacks.
\begin{figure*}[!ht]
    \vspace{2pt}

    \centering    \includegraphics[width=0.898\linewidth, trim = 53 10 40 0]{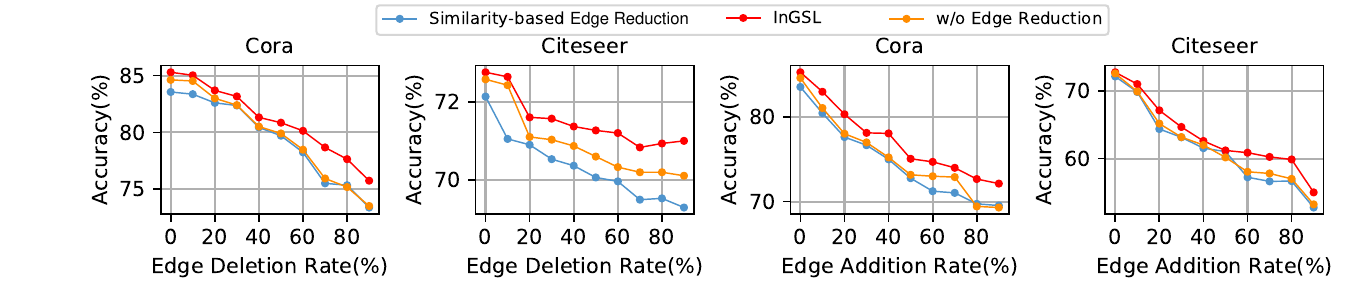}
    \caption{Robustness analysis of GRCN+InGSL with random structural noise
injection on Cora and Citeseer datasets.}
    \label{fig:robust1}
    \vspace{0pt}
\end{figure*}

\begin{figure}[t]
    \vspace{-0pt}
    \centering    \includegraphics[width=0.81\linewidth, trim = 45 15 40 10]{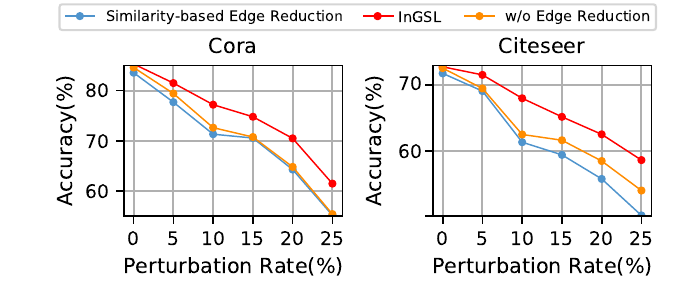}
    \caption{Robustness analysis under varying perturbation rates for graph topology attacks ($\ie$ MetaAttack~\cite{sun2020adversarial}) on the Cora and Citeseer datasets.}
    \label{fig:robust3}
    \vspace{-10pt}
\end{figure}

\section{Related Work}
\label{sec5}
\subsection{Graph Neural Networks}
Existing GNNs can be categorized into spectral and spatial methods. 
Spectral GNNs~\cite{bruna2013spectral,li2021dimensionwise,he2022convolutional,MinelloB0TC25} design signal filters based on the eigenvalues and eigenvectors of the graph Laplacian matrix to extract spectral features. 
Spatial GNNs~\cite{kipf2017semisupervised,ding2022sketch,DBLP:conf/iclr/JiangLH00CWW25} approximate graph spectral filters through a neighborhood aggregation paradigm, which aggregates and transforms features from neighbors.
However, existing GNNs often suffer from suboptimal graph quality, as real-world graphs are often noisy and dense, which degrades the performance and efficiency of GNNs~\cite{zhou2023opengsl}.
Furthermore, the neighbor aggregation scheme can be computationally inefficient for dense graphs. 
To address this limitation, some works~\cite{DiverseGraph1,DiverseGraph2} extract diverse node subsets from the original graph as a surrogate for training. 
Nevertheless, these methods still suffer from suboptimal graph quality, and fail to mitigate the issue of redundant neighbor information.
In this paper, we propose an informative graph structure learning strategy, which addresses these challenges by effectively removing redundant connections while preserving valuable edges in the structure learning process.

\subsection{Graph Sparsification}
Graph sparsification has emerged as an effective approach to reduce the computational cost for GNNs while maintaining comparable performance to the original graph.
Early methods~\cite{HermsdorffG19} relied on heuristic strategies based on graph properties (\eg node degree) to guide neighbor selection.
Mainstream approaches~\cite{chen2021unified,YouLZFL22,AdaGLT24} leverage the Graph Lottery Ticket (GLT) hypothesis, which assumes that a sparse subgraph (\ie lottery ticket) exists within the original graph and can achieve comparable performance to the complete graph. 
However, these methods cannot be directly applied to GSL models. First, they involve complex designs and are thus difficult to integrate directly into GSL methods. Furthermore, they do not explicitly consider the information diversity, resulting in poor performance when directly being used in GSL.
\subsection{Graph Structure Learning}
Early GSL~\cite{jin2020graph} methods treated the adjacency matrix as learnable parameters, resulting in significant computational overhead. 
Mainstream GSL approaches estimate edge weights based on the similarity between node embeddings, and apply similarity-based neighbor selection to remove low-similarity edges~\cite{in2024self,shen2025towards,liu2022compact}. 
Notably, a few GSL methods are not embedding-similarity–based. For instance, SEGSL~\cite{zou2023se} connects nodes by minimizing structural entropy. 
However, such methods have consistently underperformed compared with similarity-based methods~\cite{zhou2023opengsl,li2023gslb}.
Despite their superiority, we identify two limitations in embedding-similarity-based GSL methods:
First, they often require adding substantial edges in the refined graph to achieve modest improvements, greatly increasing computational costs for GNNs. Although graph sparsification methods can simplify dense structures, they either require a fixed graph to learn binary masks for edges~\cite{chen2021unified}, or neglect neighbor diversity during pruning~\cite{AdaGLT24}.
Second, these embedding-based GSL methods produce node representations that are dominated by redundant information , thereby diminishing the impact of dissimilar yet important neighbors.
In addition, some works~\cite{DSLR,Redund} have been proposed for diverse graph structure learning. 
For instance, InfoMGF~\cite{Redund} learns a fused graph structure by maximizing the mutual information between the fused graph and multiplex graphs, thereby facilitating heterogeneous graph learning.
DSLR~\cite{DSLR} extracts representative nodes from the original graph and builds edges among these nodes for graph continual learning.
However, these GSL methods still rely on similarity-based edge construction, which often results in substantial structural redundancy.
In this paper, we propose an informative graph structure learning approach (InGSL) that exploits node similarity and information diversity to eliminate redundant connections, producing a refined and compact graph structure that enhances GNN performance with limited edges.
\section{CONCLUSION}
\label{sec6}
In this paper, we reveal that existing GSL methods often produce dense graphs with substantial redundant connections and argue for the necessity of learning compact and informative graph structures.
To this end, we first conduct theoretical analyses on similarity-based edge construction strategies of existing GSL methods, and introduce InGSL as a lightweight plug‑in module designed to introduce information diversity for guiding neighbor selection.
InGSL evaluates the information diversity between each node and its candidate neighbors through an adaptively learnable mechanism, and then integrates this measure with similarity  to select neighbors.
Furthermore, InGSL uses a mutual-information-based objective to guide the learning of diversity scores, ensuring that the preserved neighbors retain information richness comparable to that of the original candidates.
Experiments demonstrate that InGSL consistently improves the performance of the learned structure across various edge reduction levels.
In the future, we aim to explore more effective strategies for accurately identifying irrelevant edges in graph structure learning.

\bibliographystyle{unsrt}
\bibliography{sample-base}

\clearpage
\appendix

\section{Proof}
\subsection{Proof of Lemma \ref{lemma1}}
\label{proofLemma1}
To prove Lemma \ref{lemma1}, we first introduce the Cauchy--Schwarz inequality below.

\begin{lemma}[Cauchy--Schwarz Inequality~\cite{bhatia1995cauchy}]
Let $\mathbf{u}$ and $\mathbf{v}$ be vectors of the same dimension. 
Then the following inequality holds:
\[
| \mathbf{u} \cdot \mathbf{v} | \leq \|\mathbf{u}\| \, \|\mathbf{v}\|,
\]
where $\|\mathbf{u}\|$ 
denotes the ${\ell}_{2}$  norm of $\mathbf{u}$.
\end{lemma}
\textit{Proof.}
Note that we simplify the GSL framework by using the same GNN for both $\text{GNN}_{T}$ and $\text{GNN}_{S}$.
For a given node $v_{i}$, let $\mathcal{N}(v_{i})$ denote the additional neighbor set of $v_{i}$ and $|\mathcal{N}(v_{i})| = N$, $ \mathbf{v}_{i} = \frac{\mathbf{E_{i}}}{||\mathbf{E}_{i}||}$ and $ \mathbf{u}_{j} = \frac{\mathbf{E_{j}}}{||\mathbf{E}_{j}||}, \forall v_{j} \in \mathcal{N}(v_{i})$, then $\mathbf{u}_j$ and $\mathbf{v}_{i}$ are unit vectors and satisfy
$\mathbf{u}_j \cdot \mathbf{v}_{i} \ge \epsilon$. The average pairwise cosine similarity $\overline{s}$ among $N(v_{i})$ is:
\begin{equation}
\overline{s} = \frac{2}{|N(v_{i})|(|N(v_{i})|-1)} \sum_{j<k} (\mathbf{u}_j \cdot \mathbf{u}_k).
\end{equation}
Let  $\mathbf{c} = \sum\limits_{j=1}^{|\mathcal{N}(v_{i})|} \mathbf{u}_j$. Then
\begin{align}
    \|\mathbf{c}\|^2 &= \sum\limits_{j=1}^{|\mathcal{N}(v_{i})|} ||\mathbf{u}_j||^{2} + 2 \sum_{j<k} (\mathbf{u}_j \cdot \mathbf{u}_k)\nonumber\\
    &=|\mathcal{N}(v_{i})| + |\mathcal{N}(v_{i})|(|\mathcal{N}(v_{i})|-1) \cdot \overline{s},
\end{align}
Since $\forall v_{j} \in \mathcal{N}(v_{i}),  \mathbf{u}_j \cdot \mathbf{v}_{i} \ge \epsilon$, we have:
\begin{equation}
\label{eq:15}
\mathbf{c} \cdot \mathbf{v}_{i} = \sum_{j=1}^{|\mathcal{N}(v_{i})|} (\mathbf{u}_j \cdot \mathbf{v}_{i}) \ge |\mathcal{N}(v_{i})| \epsilon.
\end{equation}
By the Cauchy–Schwarz inequality,
\begin{equation}
\|\mathbf{c}\|\cdot\|\mathbf{v}_{i}\| =\|\mathbf{c}\| \ge \mathbf{c} \cdot \mathbf{v}_{i} \ge |N(v_{i})| \epsilon,
\end{equation}
hence
\begin{align}
   \|\mathbf{c}\|^2 &= |N(v_{i})| + |N(v_{i})|(|N(v_{i})|-1) \cdot \overline{s}\ge|\mathcal{N}(v_{i})|^2 \epsilon^2,
\end{align}
we obtain
\begin{equation}
\overline{s} \ge \frac{N \epsilon^2 - 1}{N - 1},
\end{equation}
which completes the proof.
\subsection{Proof of Lemma \ref{lemma2}}
\label{appx:lemma2}
To prove Lemma \ref{lemma2}, we first introduce the triangle inequality below.
\begin{lemma}[Triangle Inequality~\cite{elkan2003using}]
    For any vectors \(x_{1}, x_{2}, \dots, x_{n} \in \mathbb{R}^{d}\) and non‑negative weights \(\alpha_{1}, \alpha_{2}, \dots, \alpha_{n}\) satisfying \(\sum_{k=1}^{n} \alpha_{k} = 1\), we have  
\[
\left\| \sum_{k=1}^{n} \alpha_{k} x_{k} \right\|_{2} 
\leq \sum_{k=1}^{n} \alpha_{k} \| x_{k} \|_{2}.
\]
\end{lemma} 
Since cross‑entropy loss  $\mathcal{L}(\mathbf{Z}_{i}^{(l)}\mathbf{W}_{c} , \mathbf{Y}_{i}) = -\text{log}(\mathbf{Y}_{i}\cdot \text{softmax}(\mathbf{Z}_{i}^{(l)}\mathbf{W}_{c}))$ is $\eta$-Lipschitz continuous~\cite{gao2017properties,mao2023cross}, we have:
\begin{equation}
    |\mathcal{L}(\mathbf{Z}_{i}^{(l)}\mathbf{W}_{c} , \mathbf{Y}_{i}) - \mathcal{L}(\mathbf{Z}_{i}^{(l+1)}\mathbf{W}_{c} , \mathbf{Y}_{i})| \le \eta\lVert \mathbf{Z}_{i}^{(l)}\mathbf{W}_{c} - \mathbf{Z}_{i}^{(l+1)}\mathbf{W}_{c} \rVert_{2},
\end{equation}
where $\eta=\sqrt{2}$ denotes the Lipschitz constant.
The aggregated representation of a node $v_i$ is $\mathbf{Z}_{i}^{(l+1)} = \sum_{j \in \mathcal{N}(i)} \mathbf{S}_{ij} \mathbf{Z}_{j}^{(l)}$, and $\sum_{v_j \in \mathcal{N}(v_i)} \mathbf{S}_{ij}=1$.
By the triangle inequality,  
\begin{align}
    \lVert \mathbf{Z}_{i}^{(l)}\mathbf{W}_{c} - \mathbf{Z}_{i}^{(l+1)} \mathbf{W}_{c}\rVert_{2} &=\lVert\sum_{v_j \in \mathcal{N}(v_i)} \mathbf{S}_{ij} (\mathbf{Z}_{i}^{(l)} - \mathbf{Z}_{j}^{(l)})\mathbf{W}_{c}\rVert_{2}
    \\
    & \le \sum_{j \in \mathcal{N}(i)} \mathbf{S}_{ij} \lVert(\mathbf{Z}_{i}^{(l)} - \mathbf{Z}_{j}^{(l)}) \rVert_{2}\cdot \lVert\mathbf{W}_{c}\rVert_{2}.
\end{align}
$\forall v_{j} \in \mathcal{N}(v_{i})$,the cosine similarity \(\cos(\mathbf{Z}_{i}^{(l)},\mathbf{Z}_{j}^{(l)}) \ge \epsilon\) 
 and $ \|Z^{(l)}_{i}\| \leq B $, then
\begin{equation}
    \lVert\mathbf{Z}_{i}^{(l)} - \mathbf{Z}_{j}^{(l)}\rVert_2 \leq B\sqrt{2-2\,\cos(\mathbf{Z}_{i}^{(l)},\mathbf{Z}_{j}^{(l)})} \le B\sqrt{2-2\epsilon},
\end{equation}
thus
\begin{equation}
    |\mathcal{L}(\mathbf{Z}_{i}^{(l+1)}\mathbf{W}_{c}, \mathbf{Y}_{i}) - \mathcal{L}(\mathbf{Z}_{i}^{(l)}\mathbf{W}_{c}, \mathbf{Y}_{i})| 
\leq 2 B \lVert \mathbf{W}_{c}\rVert_{2} \sqrt{1 - \epsilon}.
\end{equation}
which completes the proof.
\section{Detailed Implementation of InGSL}
\subsection{Details of the Learnable Function $f(\cdot)$ in Equation (\ref{prob})}
\label{appex:netfunc1}
For the learnable function $f(\cdot)$ in Equation (\ref{prob}), we consider two architectures: the bilinear layer and the multi‑layer perceptron (MLP).
The bilinear layer can be expressed as follows:
\begin{equation}
\label{bifunc}
f(\mathbf{E}_{i},\mathbf{E}_{j}) = \mathbf{E}_{i}\mathbf{W}_{1}\mathbf{E}_{j}^\top,
\end{equation}
where $\mathbf{E}_{i}$ and  $\mathbf{E}_{j}$ denote the embeddings of node $v_{i}$ and its neighbor $v_{j}$ obtained from Equation (\ref{embed}), and $\mathbf{W}_{1}$ denotes the learnable transformation matrix.
Equation (\ref{bifunc}) essentially computes the redundancy between node $v_{i}$ and its neighbors $v_{j}$, and then transforms this redundancy into a diversity score using the learnable matrix $\mathbf{W}_{1}$. To verify this mechanism, we set the parameters of  $\mathbf{W}_{1}$ to be non‑trainable,\ie fixed at 1, and report the results in Table \ref{tab:WI}. We observe that setting $\mathbf{W}_{1} = \mathbf{I}$ leads to a significant performance drop, which demonstrates the effectiveness of using a learnable $\mathbf{W}_{1}$ in Equation (\ref{bifunc}) .

The MLP-based learnable function can be expressed as follows:
\begin{equation}
\label{Eq:mlp}
f(\mathbf{E}_{i},\mathbf{E}_{j}) = \text{MLP}(\mathbf{E}_{i}\oplus\mathbf{E}_{j}),
\end{equation}
where $\oplus$ denotes the concatenation operation.
We experiment both architectures, and find that bilinear function (~Equation (\ref{bifunc}) ) typically achieve better results~(see Section \ref{netfunc}).
\subsection{Complexity Analysis of the InGSL Method}
\label{appex:complex}
Suppose the original constructed graph structure $\mathbf{S}$ contains $m$ edges, the edge reduction level is $r$, $n$ is the number of nodes, and the hidden dimension is $d$. The additional computational complexity introduced by InGSL is $\mathcal{O}\big(nd(d+|\mathcal{B}|+Ld+1)+m(Lrd+d+1)\big)$. Specifically, Equations (\ref{prune}) and (\ref{mask}) each have a complexity of $\mathcal{O}(m)$; Equation (\ref{prob}) has a complexity of $\mathcal{O}(nd^{2}+md)$; Equation (\ref{Eq:MIloss}) has a complexity of $\mathcal{O}\big(nd(|\mathcal{B}|+1)\big)$; and Equation (\ref{Eq:representation}) introduces additional complexity of $\mathcal{O}\big(L(n d^{2} + m r d)\big)$ when computing $\tilde{\mathbf{Z}}_{ij}$. 
Based on the above analysis, the extra complexity introduced by InGSL is lower than the computational complexity of the original GSL model.
Empirical evidence demonstrating the efficiency of InGSL is provided in Section \ref{appex:effi}.
\subsection{Details of The Mutual Information Term in Equation (\ref{Eq:MIloss})}
\label{appx:MIloss}
\(\text{sim}(\cdot)\) denotes the cosine similarity function.  
\(\mathcal{B}\) represents the set of randomly selected negative sample nodes, with \(|\mathcal{B}| < n\).  
\(\sum_{j=1}^k e^{\text{sim}(\tilde{\mathbf{Z}}_{i}, \mathbf{Z}_{j})}\) measures the cumulative similarity between \(\tilde{\mathbf{Z}}_{i}\) and the representations of different negatively sampled nodes in the original constructed graph structure \(\mathbf{S}\).  
Maximizing \(I(\tilde{\mathbf{Z}}; \mathbf{Z})\) ensures that the information provided by the selected neighbors fully captures that of all candidate neighbors, thus preventing information loss during training.
\begin{table}[!t]
	\centering
        \small
        \setlength{\abovecaptionskip}{0.cm}
	\caption{Impact of $\mathbf{W}_{1}$ on bilinear function of InGSL when integrating with SUBLIME model~\cite{liu2022towards}.}
\setlength{\tabcolsep}{0.008\textwidth}
	
{
\begin{tabular}[t!]{c|cc|cc}
        \toprule
        {Dataset} &  \multicolumn{2}{c|}{Citeseer}  &  \multicolumn{2}{c}{Blogcatlog}\\
        \midrule
        {Edge Reduction $\%$} & 30$\%$ & 
        50$\%$  & 
        30$\%$ & 
        50$\%$  
        \\
        \midrule
$\mathbf{W}_{1} = \mathbf{I}$    & 72.40±0.14 & 72.30±0.99 &  {95.04±0.19} & {94.38±0.62} \\
InGSL  & \textbf{73.05±0.33} & \textbf{73.36±0.49}  & \textbf{95.79±0.16} & \textbf{95.48±0.28}\\ 
\bottomrule
\end{tabular}
}
\label{tab:WI}
\vspace{-10pt}
\end{table}

\section{Additional Experimental Results}
\label{appex:addresult}

\subsection{Results with 70 $\%$ Edge Reduction Level}
We evaluate the InGSL performance at an edge reduction level of 70$\%$, and present the results in Table \ref{Tab:addperformance}.
\subsection{Results on the Roman-Empire Dataset}
\label{appex:roman}
We evaluate the performance of InGSL on the Roman‑empire dataset, with the results summarized in Table \ref{tab:roman}.
\subsection{Robustness Analysis \wrt Random Feature Noise.}
\label{appex:featnoise}
Figure \ref{fig:robust2} shows the performance of InGSL under different ratios of feature noise, demonstrating that InGSL can further enhance robustness to random feature masking.
\subsection{Effiency Analysis (RQ5)}
\label{appex:effi}
To evaluate the extra computational overhead that InGSL adds to GSL models, we test its time and memory efficiency on two large-scale datasets: Ogbn-Arxiv and Pubmed.
To assess time efficiency, we measure the convergence time of each algorithm, \ie the time required to achieve optimal performance on the validation set.
As shown in Table \ref{tab:effi}, InGSL incurs only a moderate increase in convergence time (29$\%$ on average) and training memory usage (30$\%$ on average) compared with the base GSL models.
These results demonstrate that InGSL remains efficient on convergence time and memory when integrated into base GSL models.
\begin{table}[!t]
	\centering
        \small
        \setlength{\abovecaptionskip}{0.cm}
        \renewcommand\arraystretch{1.}

	\caption{Convergence time and GPU memory consumption
of GSL models with InGSL on Pubmed and Ogbn-arxiv datasets.}
	\label{tab:effi}
	\resizebox{\linewidth}{!}
{
\begin{tabular}[t!]{c|cccc}
        \toprule
        {Datasets}&  \multicolumn{2}{c}{Pubmed} &  \multicolumn{2}{c}{Ogbn-arxiv} \\
        \midrule
        {} & Time(s) & 
        Memory(MB) & Time(s) & Memory(MB) )\\
        \midrule
IDGL    & 41 & 9,376 &{121} & {14,752} \\
IDGL+InGSL & {55} & {11,976} &{159} & {19,856} \\
\midrule
SUBLIME & 295 & 3,476 &{822} & {14,608} \\
SUBLIME+InGSL & {375} & 4,186 & {1,026} & {19,904}  \\
\bottomrule
\end{tabular}
}
\end{table}
\begin{table}[htp]
\renewcommand{\arraystretch}{1.2}
	\centering
    \small
    \setlength{\abovecaptionskip}{0.cm}
\setlength{\tabcolsep}{0.0025\textwidth}
\caption{Accuracy comparison(\%) of 70$\%$ edge reduction level.}
\label{Tab:addperformance}
{
\begin{tabular}[ht]{cccccc}
        \toprule
        {Dataset} &  \multicolumn{1}{c}{Cora} &  \multicolumn{1}{c}{Citeseer} & \multicolumn{1}{c}{Pubmed}  &  \multicolumn{1}{c}{Blogcatalog} & \multicolumn{1}{c}{Ogbn-arxiv} \\
        \midrule
GRCN    & 82.63 & 71.35 &{63.74} & {71.40} &OOM  \\
GRCN+InGSL & \textbf{85.07} & \textbf{72.88} &\textbf{78.61} & \textbf{76.30} &OOM  \\
\midrule
IDGL & 83.80 & 71.70 &{82.23} & {88.77} &70.05  \\
IDGL+InGSL & \textbf{84.45} & \textbf{73.10} & \textbf{82.37} & \textbf{91.79} &\textbf{70.12}  \\
\midrule
SUBLIME    & 81.41 & 71.75 &{75.97} & {92.19} &{69.42} \\
SUBLIME+InGSL & \textbf{83.10} & \textbf{73.16} &\textbf{79.80} & \textbf{95.24} &\textbf{70.56} \\
\midrule
RWGSL &{82.39} &{71.63} & {78.86} & {74.45} &{69.77}  \\
RWGSL+InGSL & \textbf{84.10} & \textbf{72.25} & \textbf{79.57} & \textbf{92.42} &\textbf{70.09}  \\
\midrule
UnGSL & 84.21 & 72.61  &{76.64} & {72.43} &OOM  \\
UnGSL+InGSL & \textbf{84.55} & \textbf{72.80} & \textbf{78.57} & \textbf{73.60} &OOM  \\
\midrule
PROSE & 81.57 & 71.91  &{82.65} & {75.21} &67.47  \\
PROSE+InGSL & \textbf{81.62} & \textbf{72.70} & \textbf{82.91} & \textbf{76.51} &\textbf{69.87}  \\
\bottomrule
\end{tabular}
}
\end{table}

\begin{table}[!t]
	\centering
        \small
        \setlength{\abovecaptionskip}{0.cm}
	\caption{Performance Comparison Between InGSL and similarity-based neighbor selection on the Roman-Empire dataset.}
\setlength{\tabcolsep}{0.008\textwidth}
	
{
\begin{tabular}[t!]{c|ccc}
        \toprule
        {Edge Reduction $\%$} & 30$\%$ & 
        50$\%$  & 
        70$\%$  
        \\
        \midrule
GRCN    & 43.24±0.05& 43.00±0.51 &  {31.20±1.10}  \\
GRCN+InGSL  & \textbf{45.35±0.15} & \textbf{45.92±0.94}  & \textbf{32.34±0.54}\\ 
\midrule
IDGL   & 60.46±0.24 & 59.86±0.19 &  {59.32±0.00}  \\
IDGL+InGSL  & \textbf{64.36±0.69} & \textbf{64.73±0.51}  & \textbf{63.53±1.55} \\ 
\midrule
SUBLIME    & 64.59±0.69 & 64.47±0.35 &  {64.26±0.29}\\
SUBLIME+InGSL  & \textbf{64.94±0.25} & \textbf{64.86±0.49}  & \textbf{65.51±0.36} \\ 
\bottomrule
\end{tabular}
}
\label{tab:roman}
\vspace{0pt}
\end{table}
\begin{figure}[ht]
    \vspace{0pt}
    \centering    \includegraphics[width=0.8\linewidth, trim = 45 15 40 10]{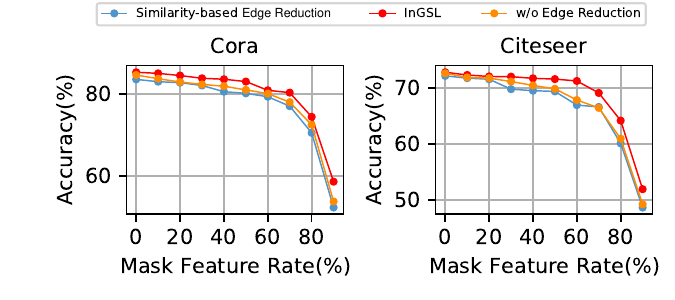}
    \caption{Robustness analysis with random feature noise
injection on Cora and Citeseer dataset.}
    \label{fig:robust2}
    \vspace{-0pt}
\end{figure}

\end{document}